\documentclass{article}

\usepackage{arxiv}
\usepackage{hyperref}       % hyperlinks
\usepackage{url}            % simple URL typesetting
\usepackage{booktabs}       % professional-quality tables
\usepackage{amsfonts}       % blackboard math symbols
\usepackage{nicefrac}       % compact symbols for 1/2, etc.
\usepackage{microtype}      % microtypography
\usepackage{lipsum}
\usepackage{float}
\usepackage{graphicx}
\usepackage{caption}
\usepackage{subcaption}
\usepackage{amsthm}
\usepackage{amsmath,bm}
\usepackage[noabbrev,capitalize,nameinlink]{cleveref}
\usepackage{multirow}
\usepackage{comment}
\usepackage{notoccite}

\usepackage[shortlabels]{enumitem}

\usepackage[round]{natbib}

\graphicspath{ {./figs/} }

\newtheorem{proposition}{Proposition}[]

\newcommand{\xb}{\mathbf{x}}

\DeclareMathOperator*{\argmin}{\arg\!\min}

\newcommand{\STAB}[1]{\begin{tabular}{@{}c@{}}#1\end{tabular}}
\title{On the Role of Model Uncertainties in Bayesian Optimization}

\author{
Jonathan Foldager \\
 Section for Cognitive Systems, DTU Compute \\
 Technical University of Denmark \\
 Correspondence: \texttt{jonf@dtu.dk} 
 \And
  Mikkel Jordahn \\
  Section for Cognitive Systems, DTU Compute \\
  Technical University of Denmark \\
  \And
 Lars Kai Hansen \\
  Section for Cognitive Systems, DTU Compute \\
  Technical University of Denmark \\
  \And
 Michael Riis Andersen \\
  Section for Cognitive Systems, DTU Compute \\
 Technical University of Denmark \\
}

\begin{document}

\maketitle
\begin{abstract}
Bayesian optimization (BO) is a popular method for black-box optimization, which relies on uncertainty as part of its decision-making process when deciding which experiment to perform next. However, not much work has addressed the effect of uncertainty on the performance of the BO algorithm and to what extent calibrated uncertainties improve the ability to find the global optimum. In this work, we provide an extensive study of the relationship between the BO performance (regret) and uncertainty calibration for popular surrogate models and compare them across both synthetic and real-world experiments. Our results confirm that Gaussian Processes are strong surrogate models and that they tend to outperform other popular models. Our results further show a positive association between calibration error and regret, but interestingly, this association disappears  when we control for the type of model in the analysis. We also studied the effect of re-calibration and demonstrate that it generally does not lead to improved regret. Finally, we provide theoretical justification for why uncertainty calibration might be difficult to combine with BO due to the small sample sizes commonly used.
\end{abstract}

%\keywords{Bayesian Optimization \and Uncertainty Calibration}
%\twocolumn

\section{Introduction}
Probabilistic machine learning provides a framework in which it is possible to reason about uncertainty for both models and predictions \citep{ghahramani2015probabilistic}. 
It is often argued that especially in high-stakes applications (healthcare, robotics, etc.), uncertainty estimates for decisions/predictions should be a central component and that they should be well-calibrated \citep{kuleshov2022calibrated}. 
The intuition behind calibration is that the uncertainty estimates should accurately reflect the reality; for example if a classification model predicts 80\% probability of belonging to class $A$ on 10 datapoints, then (on average) we would expect 8 of those 10 samples actually belong to class $A$. 
Likewise -- but less intuitively -- in regression, if a calibrated model generates a prediction $\mu$ and uncertainty estimate $\sigma$, we would see $p$ percent of the data lying inside a $p$ percentile confidence interval of $\mu$ \citep{busk2021calibrated}. 
In general, uncertainty can be divided into \textit{aleatoric} (irreducible inherent randomness in the data-generating process) and \textit{epistemic} (lack of knowledge, i.e. it can be reduced if more data is collected) \citep{hullermeier2021aleatoric}. 
However, this distinction is rarely used when evaluating uncertainty estimates for regression tasks and although this has been critiqued \citep{sluijterman2021evaluate}, it is highly non-trivial to achieve for real-world applications because it normally requires access to the underlying true function.

Uncertainty also plays a central role Bayesian Optimization (BO) \citep{snoek2012practical}, which will be the focus of this paper. As a sequential design strategy for global optimization, BO has several applications with perhaps the most popular ones being general experimental design \citep{shahriari2015taking} and model selection for machine learning models \citep{bergstra2011algorithms}. 
%BO uses a combination of the prediction of underlying the objective function and corresponding uncertainty estimates from probabilistic machine learning models in order to select which input combinations to be evaluated next such that it eventually finds the global minimum or maximum. 
BO is most often used when the objective function is expensive (e.g. monetary, time consuming, or ethically) to evaluate, gradients between in- and outputs are not available, and/or data acquisition is limited to few training samples \citep{agnihotri2020exploring}. 
A BO protocol works by iteratively fitting a probabilistic surrogate model to observed values of an objective function, and using a so-called acquisition function based on the surrogate model, to select where to query the objective function next. 
In acquisition functions, there is an inherent trade-off between exploring input areas in which the surrogate model is uncertain of the underlying objective function, and exploiting areas where the surrogate model already knows that the objective value is high.
As such, it seems obvious that in order for this exploration-exploitation trade-off to be good, the probabilistic model must be well calibrated.
It is, however, still not well-described how much calibration actually affects BO procedures.
One could imagine that if calibration leads to a better model representation of the underlying objective function, as would be the general intuition, it would be natural to expect that improving calibration via so-called \textit{re-calibration} \citep{kuleshov2018accurate} will aid in finding the global optimum of that same function.

\subsection{Our Contribution}
In this paper, we set out to investigate how the model uncertainties affect BO performance by means of both numerical and theoretical perspectives. Our work is highly motivated by the general intuition and understanding in the community that BO surrogate models with better / well-calibrated uncertainty estimates will perform better (i.e. reach better final and/or total regret). In particular, our paper is concerned with studying statements such as "BO crucially relying on calibrated uncertainty estimates" \citep{springenberg2016bayesian} and that methods performing worse "due to their frequentist uncertainty estimates" \citep{deshwal2021bayesian}. But how well-calibrated do we need to be in order to achieve good BO performance? In order to investigate these questions, we provide three major contributions:

\begin{itemize}
    \item An extensive study of commonly used surrogate models by tracking their calibration errors and regrets to see if the best calibrated models also perform best in BO. 
    \item An investigation on how re-calibration and simple manipulation of model uncertainties (which essentially controls the general predictive sharpness and together with the mean its calibration) affect BO performance. If well-calibrated uncertainties leads to good BO performance, how much can we manipulate those same uncertainties and still get decent BO performance?
    \item Numerical and theoretical results to substantiate a discussion on the role of calibration in BO. Especially on the relationship between the number of re-calibration samples and the variance of the calibration curve. This is to study if few validation samples --- as is most often the case when doing BO --- provide significant and robust changes in the surrogate model.
\end{itemize}

\begin{figure*}[t!]
    \begin{subfigure}[b]{0.49\textwidth}
    \includegraphics[width=\textwidth]{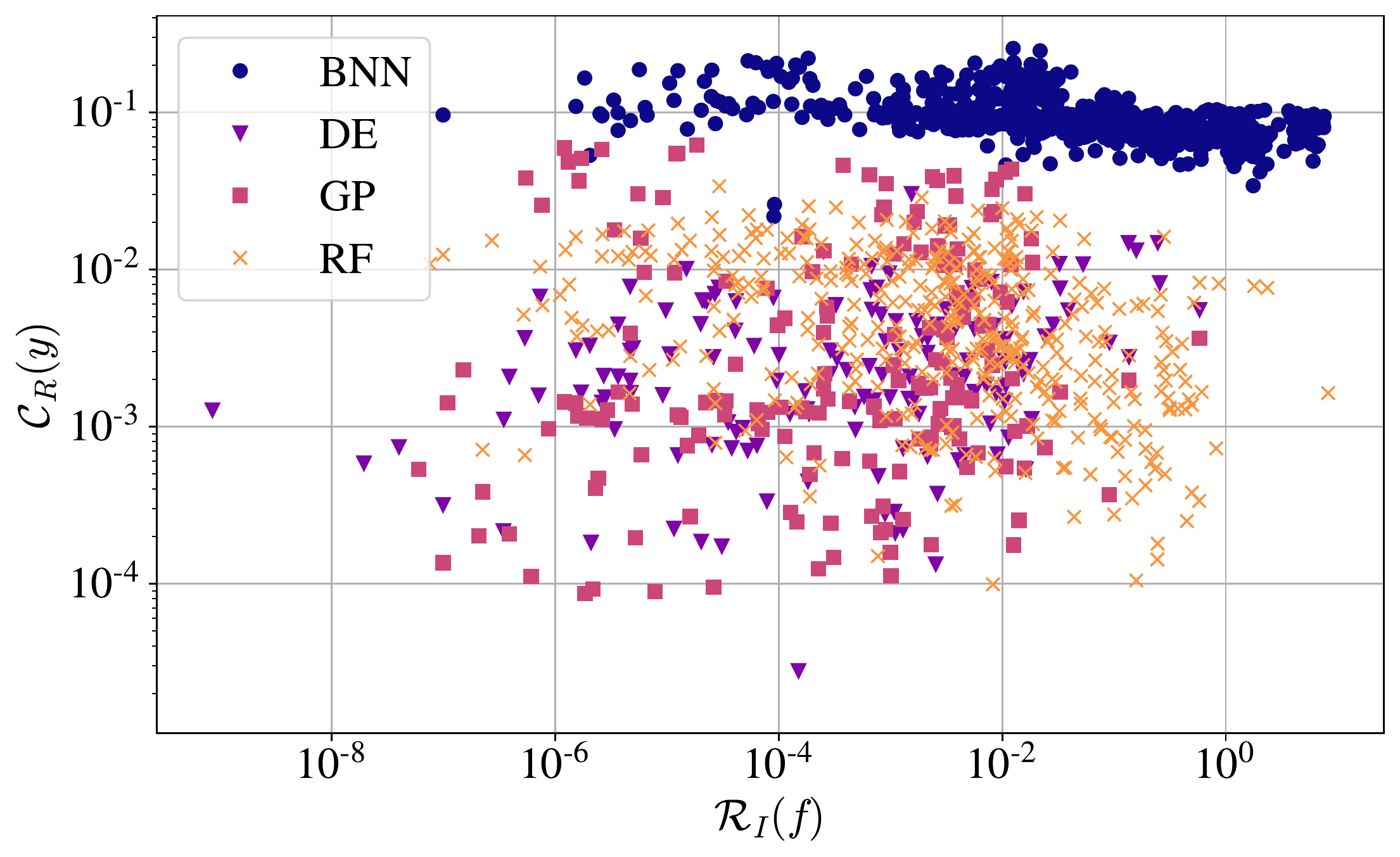}
    \caption{Test calibration after surrogates are trained on i.i.d. samples. }
    \end{subfigure}
    \begin{subfigure}[b]{0.49\textwidth}
    \includegraphics[width=\textwidth]{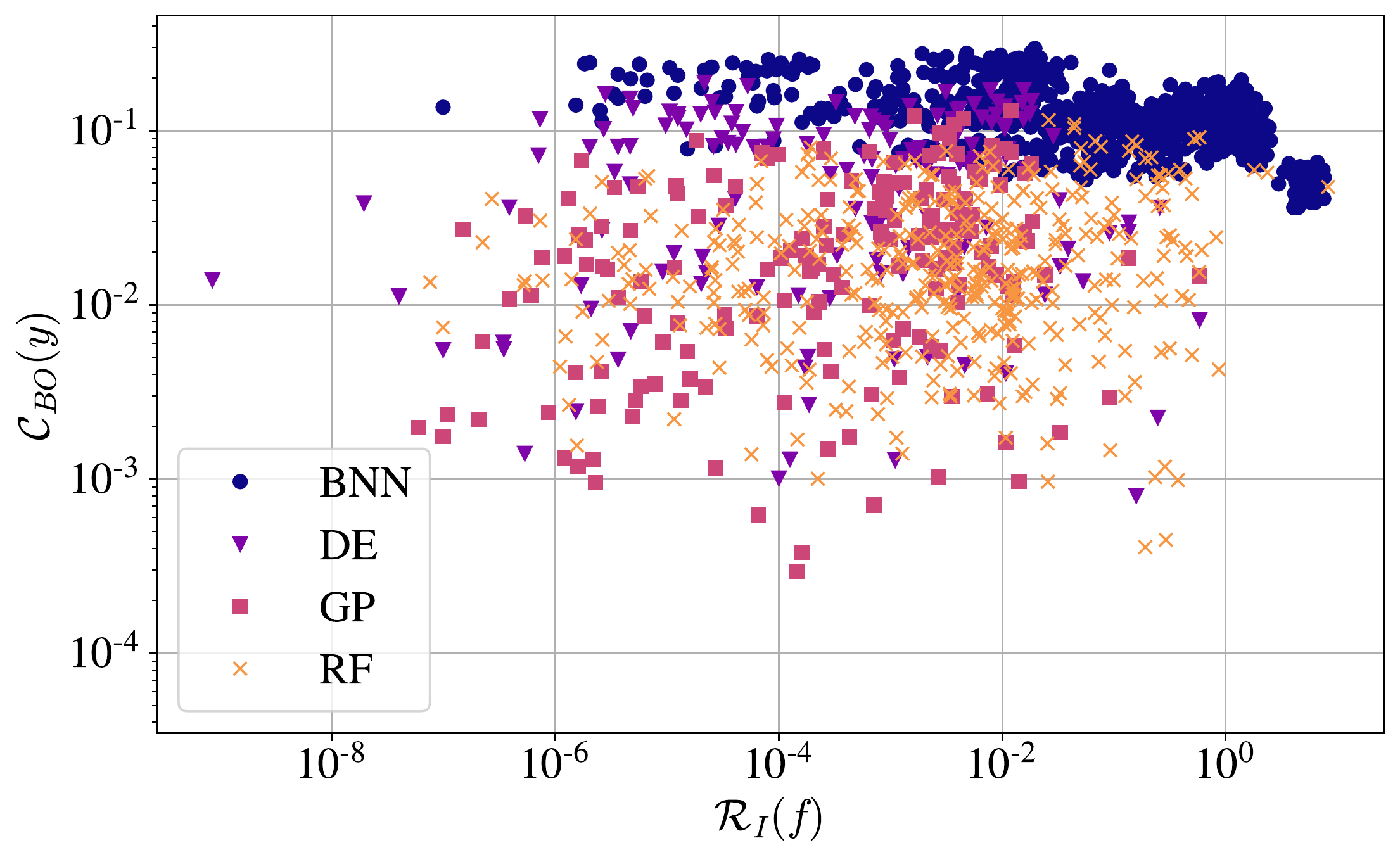}
    \caption{Test calibration after surrogates are trained on BO samples. }
    \end{subfigure}
    \caption{Regret vs. test calibration MSE \label{fig:regret-calibration-correlation}. (a) shows the calibration error (after regression) as a function of regret. (b) shows calibration error (after BO) as a function of regret. $C_R(y)$ = test calibration MSE on $y$ from model trained on 100 i.i.d. samples. $C_{BO}(y)$ = test calibration MSE from model trained on 10 i.i.d. + 90 BO samples. $R_{I}(f)$ = instant regret of the function value $f$ in the last BO iteration. }
\end{figure*}
\newpage
\subsection{Related Work}
A great deal of work has been carried out for uncertainty calibration for regression models \citep{kuleshov2018accurate,song2019distribution,ovadia2019can,busk2021calibrated,nado2021uncertainty} and the useful uncertainty toolbox \citep{chung2021uncertainty} makes it easy to assess the calibration level of various models. In the very recent work by \citet{deshpande2021calibration}, a procedure for calibrating  GPs during BO was proposed. Given the small sample sizes available in BO, the idea is to use leave-one-out cross-validation and utilize the calibration algorithm proposed in earlier work by \citet{kuleshov2018accurate}. We note that potential issues might arise from this procedure as the earlier work by \citep{kuleshov2018accurate} state multiple times their approach produces calibrated forecasts "\textit{given enough i.i.d. data}". However, the data available during BO is rarely large nor independent and identically distributed (i.i.d.), and the  goal of our work is to dive deeper into this. Other research on the role of uncertainty calibration include examples such as the work by \citet{bliznyuk2008bayesian}, where the authors propose a way of using Markov Chain Monte Carlo (MCMC) to get calibrated predictions for GPs. In work the by \citet{belakaria2020uncertainty}, the authors investigate uncertainty aware multi-objective (multidimensional output) BO and argue that due to the uncertainty incorporating strategy, their model outperforms state-of-the-art procedures.

\section{Background}
Bayesian Optimization (BO) is concerned with the optimization task of finding the global minimum $\xb^*=[x_1^*,x_2^*,...,x_D^*]^\top$ of some objective function $f(\xb)$, where $\xb$ is a $D$-dimensional vector, i.e.
\begin{equation}
    \xb^* = \argmin f(\xb).
\end{equation}
We assume that the optimization objective $f(\xb) \in \mathbb{R}$ is contaminated with noise, i.e. it is only possible to observe $y(\xb) = f(\xb) + \epsilon$, where $\epsilon$ is additive noise often assumed to follow an isotropic normal distribution. 
In many scenarios such as hyperparameter tuning of neural networks, the set of input variables $\xb$ are rarely all continuous and often no closed-form expression for $f$ exists. Hence, BO is well-suited when $f$ is a so-called "black-box" function \citep{turner2021bayesian}. 
At least two crucial decisions are to be made when using BO in practice: 1) the choice of surrogate model, which is to mimic/learn the underlying objective function $f$, and 2) the acquisition function (AF), which controls the strategy for deciding which input $\xb$ to sequentially pick. Popular choices for surrogate models include Gaussian Processes \citep{rasmussen2003gaussian, snoek2012practical} and Random Forests \citep{bergstra2011algorithms}, but in principle any probabilistic model, e.g. Deep Ensembles \citep{lakshminarayanan2017simple} or mean-field Bayesian Neural Networks \cite{springenberg2016bayesian}, can be used. For the choice of AF, \textit{Expected Improvement} (EI) as proposed by \citet{Jones1998-cu} is often used and is defined as follows
\begin{equation}
            \text{EI}(\xb) =
            (\mu(\xb) - f(\xb^+))\Phi(Z)+\sigma(\xb)\phi(Z), \label{eq:ei}
\end{equation}
if $\sigma(\xb) > 0$ otherwise $\text{EI}(\xb) = 0$, and with:
\begin{equation}
            Z(\xb) =
            \begin{cases}
            \frac{\mu(\xb) - f(\xb^+)}{\sigma(\xb)} &\text{if}\ \sigma(\xb) > 0 \\
            0 & \text{if}\ \sigma(\xb) = 0
        \end{cases},
\end{equation}

where $\mu(\xb)$ and $\sigma(\xb)$ denote the mean and standard deviation, respectively, of the surrogate function at $\xb$, $f(\xb^+)$ denotes the best function value observed so far, and $\Phi$ and $\phi$ denote the cumulative distribution function (CDF) and probability density function (PDF) of a standard normal distribution, respectively. As mentioned in the introduction, the uncertainty estimate $\sigma(\xb)$ plays a key role in BO including in this AF: large uncertainties will generally cause the second term in \cref{eq:ei} to dominate, and vice versa, when the uncertainty is relatively small, the first term will dominate. Intuitively, calibrated uncertainties would lead to more accurate estimates of expected improvements, which in turn should lead to a more data-efficient BO protocol.

\begin{table*}[t!]
\caption{Average ($\pm$ std.) metrics on benchmark data across 10 different input dimensions, 10 random initializations and 10 different problems for each dimension yielding 1000 experiments per estimate. No calibration error when doing regression is computed for the recalibrated models, since recalibration only was done during BO. Notation: $\mathcal{R}_f^I$ = instant regret on $f$ in the last BO iteration, $\mathcal{R}_f^T$ = total regret on $f$, $E_{C_y}^{R}$ = test calibration MSE on $y$ from model trained on 100 i.i.d. samples, $E_{C_y}^{BO}$ = test calibration MSE from model trained on 10 i.i.d. + 90 BO samples, $\rho_{R}$ = correlation coefficient between $\mathcal{R}_f^I$ and $E_{C_y}^{R}$, $\rho_{BO}$ = correlation coefficient between $\mathcal{R}_f^I$ and $E_{C_y}^{BO}$. $^*= p<0.05$, $^{**} = p<0.0001$ (Bonferroni corrected), BNN= Bayesian Neural Network, DE = Deep Ensemble, GP = Gaussian Process, RF = Random Forest, RS = Random Search.  \label{correlation_table}} 
\begin{center}
\begin{tabular}{llcccccc}
 & Surrogate & $\mathcal{R}_f^I$ &  $\mathcal{R}_f^T$ &  $E_{C_y}^{R}$ & $E_{C_y}^{BO}$ & \textbf{$\rho_{R}$} & \textbf{$\rho_{BO}$}\\
\toprule 
\multirow{6}{*}{\STAB{\rotatebox[origin=c]{90}{Baseline}}}& BNN &$0.5614\,\,(\pm 1.10)$&$56.3\,\,(\pm 103.9)$&$0.089\,\,(\pm 0.03)$&$0.119\,\,(\pm 0.05)$&$-0.23^{**}$&$-0.32^{**}$\\
&DE&$0.0022\,\,(\pm 0.02)$&$10.7\,\,(\pm 19.3)$&$0.004\,\,(\pm 0.01)$&$0.024\,\,(\pm 0.03)$&$0.07$&$0.02$\\
&GP&$\mathbf{0.0013}\,\,(\pm 0.02)$&$\mathbf{6.8}\,\,(\pm 16.0)$&$\mathbf{0.003}\,\,(\pm 0.01)$&$\mathbf{0.017}\,\,(\pm 0.02)$&$0.02$&$0.02$\\
&RF&$0.0316\,\,(\pm 0.29)$&$18.5\,\,(\pm 44.6)$&$0.005\,\,(\pm 0.01)$&$0.039\,\,(\pm 0.03)$&$-0.04$&$-0.003$\\
&RS&$0.3010\,\,(\pm 0.67)$&$42.3\,\,(\pm 84.5)$&-&-&-&-\\
&All&-&-&-&-&$0.28^{**}$&$0.16^{**}$\\
\hline
\multirow{4}{*}{\STAB{\rotatebox[origin=c]{90}{Recal.}}}& BNN &$0.5858\,\,(\pm 1.14)$&$55.5\,\,(\pm 101.7)$&-&$0.078\,\,(\pm 0.028)$&-&-\\
&DE&$0.0038\,\,(\pm 0.03)$&$12.4\,\,(\pm 21.5)$&-&$0.017\,\,(\pm 0.012)$&-&-\\
&GP&$\mathbf{0.0012}\,\,(\pm 0.02)$&$\mathbf{8.9}\,\,(\pm 28.5)$&-&$\mathbf{0.012}\,\,(\pm 0.012)$&-&-\\
&RF&$0.0396\,\,(\pm 0.32)$&$19.8\,\,(\pm 49.1)$&-&$0.014\,\,(\pm 0.013)$&-&-\\
\bottomrule
\end{tabular}
\end{center}
\end{table*}

\paragraph{Uncertainty Calibration}
Following the work by \citet{kuleshov2018accurate}, a regression model is well-calibrated if approximately $q$ percent of the time test samples fall inside a $q$ percent confidence interval of the predictive distribution. Thus the predictive distribution should be close to the true data distribution. For regression tasks, the model calibration can be assessed using the mean square error $E_{C_y} = \frac{1}{P} \sum_p w_p (C_y(p) - p)^2$ for calibration given by
\begin{equation}\label{eq:Cy}
    C_y(p) = \frac{1}{N_T} \sum_{t=1}^{N_T} \mathbb{I} [ y_t  \leq F_t^{-1}(p) ],
\end{equation}
where $F_t^{-1}$ is the quantile function, i.e. $F_t^{-1}(p)  \equiv \inf\limits_y \{ y \: | \: p \leq F_t(y) \},$ for the $t$'th datapoint evaluated at percentile $p$, $\mathbb{I}$ is an indicator function and $w_p$ can be chosen to adjust the importance of percentiles with fewer datapoints. Throughout this paper, we assume $w_p = 1 \,\, \forall \,p$. The closer $E_{C_y}$ is to zero, the better calibrated the model is. 

\paragraph{Recalibration}
\citet{kuleshov2018accurate} also proposes a general procedure for recalibrating any model. A so-called recalibrator model $R$ is trained on independent and identically distributed (i.i.d.) validation set and subsequently, applied to readjust the CDF of the model's predictive distribution $F_t$ for some observation $y_t$, i.e. the recalibrated predictive distribution is $R \circ F_t$. This is done via learning an isotonic mapping: $R: \left[0, 1\right] \rightarrow \left[0, 1\right]$ from the predicted probabilities of events of the form $\left(-\infty, y_t\right]$ to the corresponding empirical probabilities. See Alg. 1 in \citet{kuleshov2018accurate} for more details.
%
%\begin{equation}
%    R \circ F \xrightarrow{} C_y(p).
%\end{equation}
After training the recalibrator model $R$, the relevant summary statistics (e.g. moments and intervals) of the re-calibrated distributions can be computed numerically from $R \circ F_t$.

\section{Experiments}
In this section, we describe a collection of numerical experiments designed to study the relationship between calibration and regret. We focus our study on four popular models, namely Gaussian Processes (GPs), Random Forests (RFs), Deep Ensembles (DEs) and mean-field Bayesian Neural Networks (BNNs). For GPs, DEs and BNNs, we assume an isotropic Gaussian likelihood and for RFs, we impose a Gaussian predictive distribution, where the mean and variance are estimated from the tree predictions. Our experiments are based on both synthetic and real-world data: for experiments with synthetic data, we use the common benchmark suites for optimization called Sigopt \citep{jamil2013literature, dewancker2016stratified} and for the real-world data, we apply BO to hyperparameter tuning of neural networks for classifying the MNIST dataset \citep{mnist}.

\paragraph{Synthetic Data Experiments}
We randomly sample ten problem instances (i.e. different objective functions) from the benchmark suite in each dimension ($ D \in \{ 1, 2, .., 10\}$). For each problem, we repeat the experiment ten times using different random initialization of the BO routines. A total of one thousand experiments per surrogate function is thus conducted. Ten initial samples followed by 90 iterations is used consistently throughout the experiments.
Normal distributed noise is added to the objective functions and all experiments are carried out using the EI acquisition function specified in  \cref{eq:ei}. 

Our key performance metrics are regret, calibration error, sharpness as defined in the following.
We report the calibration error $E_{C_y}$ for the output variable $y = f + \epsilon$ as being the mean squared calibration error evaluated on a large i.i.d. test set ($N_{\text{test}} = 1000$) as
\begin{equation}\label{eq:cal_error}
    E_{C_y} = \frac{1}{P} \sum_{j=1}^P (C_y(p_j) - p_j)^2,
\end{equation}
where $C_y(p_j)$ is defined in eq. \eqref{eq:Cy} and for $0\leq p_1 \leq p_2 ... \leq p_P \leq 1$ as suggested by \citet{kuleshov2018accurate}. We use $P=20$ with equidistant $p_j$ values and quantify the BO performance on synthetic experiments using the instantaneous regret metric measured after the last BO iteration,
\begin{align}
    \mathcal{R}_f^I = f_{\text{min}} - f(x^*_T),
\end{align}
where $f(x)$ is the true underlying function, $f_{\text{min}} \equiv \min\limits_x f(x)$ is function value at the true global minimum, and $x^*_T \equiv \arg\min_{x_t} \{ y(x_t) \}_{t=1}^T$ is the input value for the best observation after $T$ iterations. The subscript $f$ in $\mathcal{R}_f^I$ indicates that we compute the regret using the true function, which is possible for synthetic data, whereas for real-data, we will use the observed values to estimate the regret denoted by $\mathcal{R}_y^I$.
%subscript $f$ refers to regret on the noiseless observations as given by the difference between the best found $y$ corresponding $f$ value and the overall $f_{\text{min}}$.
We also report total regret defined as $\mathcal{R}_f^T = \sum_{i=1}^T \left[f_{\text{min}} - f(x^*_i)\right],$
%\begin{align}
%    \mathcal{R}_f^T = \sum_{i=1}^T \left[f_{\text{min}} - f(x^*_i)\right].
%\end{align}
using $T=90$ in all experiments, yielding 100 datapoints for each experiment. Finally, all regret values are reported after standardizing objective function values and we report the sharpness as the average negative entropy of the predictive distributions.
See further experimental details in the appendix.  Code will be released on GitHub along with the camera-ready version.

\paragraph{Correlation between calibration error and regret}
Our first experiment is designed to compare calibration and regret for all models. The results are summarized in \cref{correlation_table}. First, we compare the BO performance for all surrogate models by computing the total and instant regret for all models for all problem instances. It is seen in the table that GPs outperform all methods wrt. both types of regret, but are only slightly better than Deep Ensembles. It is also noted that the mean-field BNN is performing substantially worse than the other models, including uniform random sampling.

We quantify the calibration for each surrogate model in two ways: i) in a pure regression setting with i.i.d. data ($E_{C_y}^R$) and ii) using the samples collected during BO ($E_{C_y}^{BO}$) using the same sample size $N = 100$. \cref{correlation_table} shows that GPs, RFs, and DEs all achieve comparable calibration error with the GPs being marginally better. Again, all three methods clearly outperform the mean-field BNN in both metrics.

\begin{table*}[t!]
\caption{Average ($\pm$ std.) metrics on hyperparameter tuning of neural networks classifying MNIST data. We tested linear correlation between regret and calibration levels across ten random seeds, and found no significant correlation across models. Notation: $\mathcal{R}_y^I$ = instant regret on $y$ in the last BO iteration, $\mathcal{R}_y^T$ = total regret on $y$, $E_{C_y}$ = test calibration MSE on $y$, $\mathcal{S}$ = mean predictive posterior sharpness. BNN= Bayesian Neural Network, DE = Deep Ensemble, GP = Gaussian Process, RF = Random Forest, RS = Random Search.  \label{mnist_table}} 
\begin{center}
\begin{tabular}{llcccc}
 & & $\mathcal{R}_y^I$ &  $\mathcal{R}_y^T$ &  $E_{C_y}$ &  $\mathcal{S}$ \\
\toprule 
\multirow{5}{*}{\STAB{\rotatebox[origin=c]{90}{Baseline}}}&BNN&$0.064\,\,(\pm 0.044)$ & $10.811\,\,(\pm 3.578)$ & $0.091\,\,(\pm 0.004)$ & $0.511\,\,(\pm 0.140)$ \\
&DE&$\mathbf{0.003}\,\,(\pm 0.003)$ & $\mathbf{7.161}\,\,(\pm 2.028)$ & $0.052\,\,(\pm 0.020)$ & $-0.176\,\,(\pm 0.150)$ \\
&GP&$0.005\,\,(\pm 0.003)$ & $8.023\,\,(\pm 2.404)$ & $\mathbf{0.019}\,\,(\pm 0.021)$ & $-0.388\,\,(\pm 0.407)$ \\
&RF&$0.013\,\,(\pm 0.004)$ & $7.676\,\,(\pm 2.244)$ & $0.022\,\,(\pm 0.015)$ & $0.082\,\,(\pm 0.373)$ \\
&RS&$0.022\,\,(\pm 0.011)$ &-&-&-\\
\hline
\multirow{4}{*}{\STAB{\rotatebox[origin=c]{90}{Recal.}}}&BNN&$0.160\,\,(\pm 0.056)$ & $17.823\,\,(\pm 5.266)$ & $0.077\,\,(\pm 0.008)$ & $0.262\,\,(\pm 0.282)$ \\
&DE&$\mathbf{0.002}\,\,(\pm 0.002)$ & $\mathbf{6.781}\,\,(\pm 1.621)$ & $0.064\,\,(\pm 0.007)$ & $0.256\,\,(\pm 0.090)$ \\
&GP&$0.005\,\,(\pm 0.003)$ & $6.894\,\,(\pm 2.974)$ & $\mathbf{0.041}\,\,(\pm 0.014)$ & $0.001\,\,(\pm 0.455)$ \\
&RF&$0.011\,\,(\pm 0.006)$ & $7.250\,\,(\pm 1.843)$ & $0.006\,\,(\pm 0.006)$ & $-0.677\,\,(\pm 0.216)$ \\
\bottomrule
\end{tabular}
\end{center}
\end{table*}
\cref{correlation_table} also shows the Pearson correlation between the calibration error and instantaneous regret (last two columns). 
The correlation coefficient between $E_{C_y}^R$ and $\mathcal{R}_f^I$ when computed across all surrogate models is $0.28$ ($p<10^{-5}$) and similarly, the correlation coefficient between $E_{C_y}^{BO}$ and $\mathcal{R}_f^I$ is $0.16$ ($p<10^{-5}$). 
We also report the two correlation metrics conditioned on the type of model. Interestingly, the resulting correlations are much weaker and generally not statistically significant, leading to an instance of Simpson's paradox \citep{Wagner1982SimpsonsPI}. Again the mean-field BNN is an exception, which actually shows a significant negative correlation. 
These results suggest that it is not necessarily the degree of calibration, but the surrogate model that drives the performance. \cref{fig:regret-calibration-correlation} shows a scatter plot of the same data and here is seen that conditioned on a given model, the trend disappears.
Finally, we also apply a leave-one-out recalibration procedure \citep{kuleshov2018accurate} and observe, surprisingly, that this does not improve regret in general.

\paragraph{Hyperparameter tuning on neural networks}
Next, we perform a similar experiment on a real-world BO problem by means of hyperparameter tuning of (fully connected, single hidden layer) neural networks for classifying MNIST digits. We create a grid of the hyperparameters epochs ([$1$, $10$], step size 1), dropout rate ([$0$, $0.8$], step size 0.08), learning rate
([$10^{-5}$, $0.1$] equally spaced on a logarithmic scale), hidden layer ([$1$, $300$], step size 30) and batch size ([$8$, $256$] step size 32).

Looking at \cref{mnist_table}, we observe several interesting findings. 
First, GPs and DEs look superior to the other methods, but also the least sharp. 
We also observe that the recalibration procedure does not appear to improve the regret significantly for any model (which is consistent with our synthetic experiments), but the calibration error on an independent large test set is actually worse for DE and GP after recalibration. 
In \cref{sec:discussion}, we provide a discussion founded in theory of why this might be the case. However, recalibration seems to on average have a sharpening effect on all models except for the RF, where it "softens" its posterior variance.

\begin{figure*}[t!]
    \begin{subfigure}[b]{0.33\textwidth}
    \includegraphics[width=\textwidth]{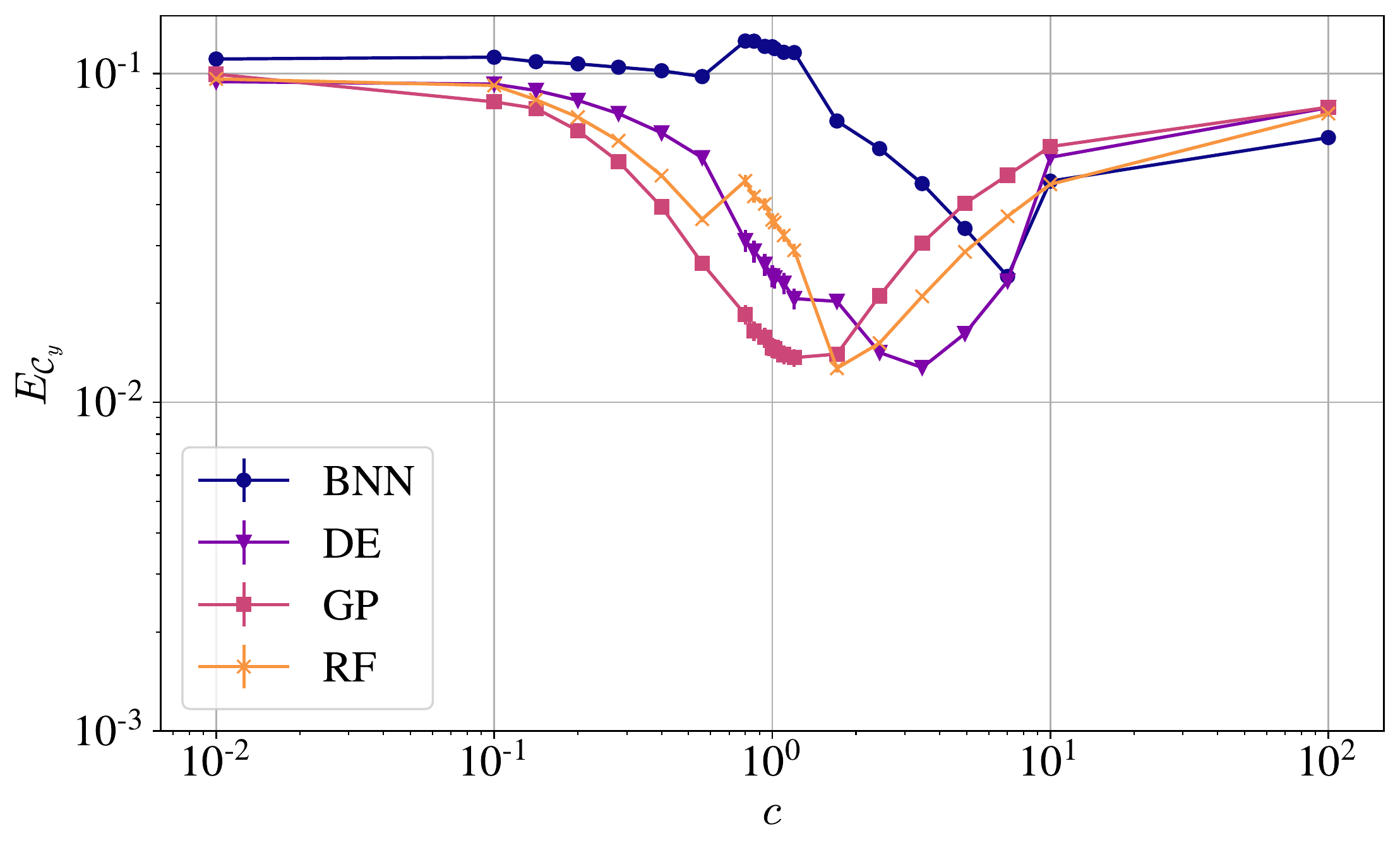}
    \caption{Calibration after BO protocol}
    \end{subfigure} 
    \begin{subfigure}[b]{0.33\textwidth}
    \includegraphics[width=\textwidth]{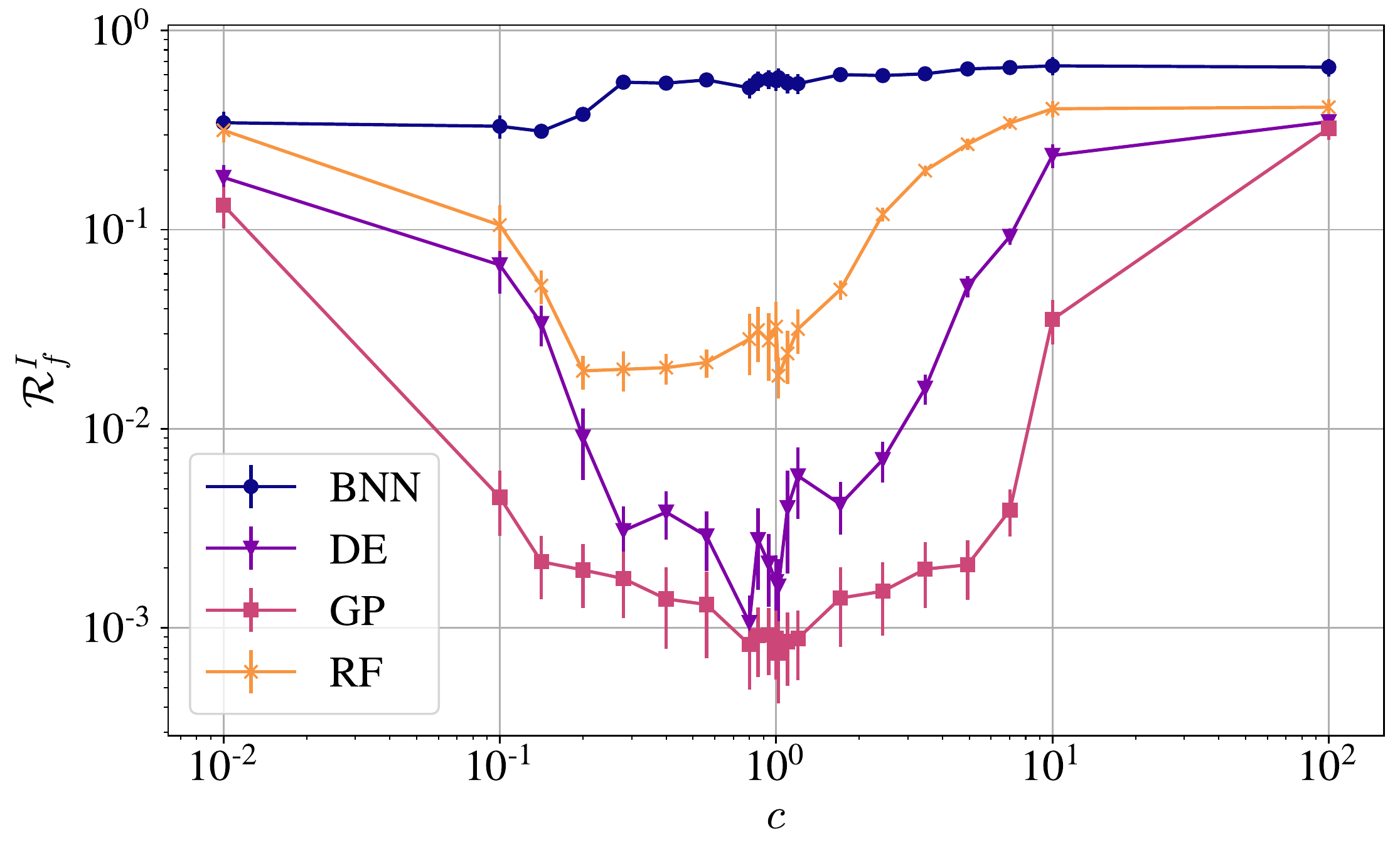}
    \caption{Regret after BO protocol}
    \end{subfigure}
    \begin{subfigure}[b]{0.33\textwidth}
    \includegraphics[width=\textwidth]{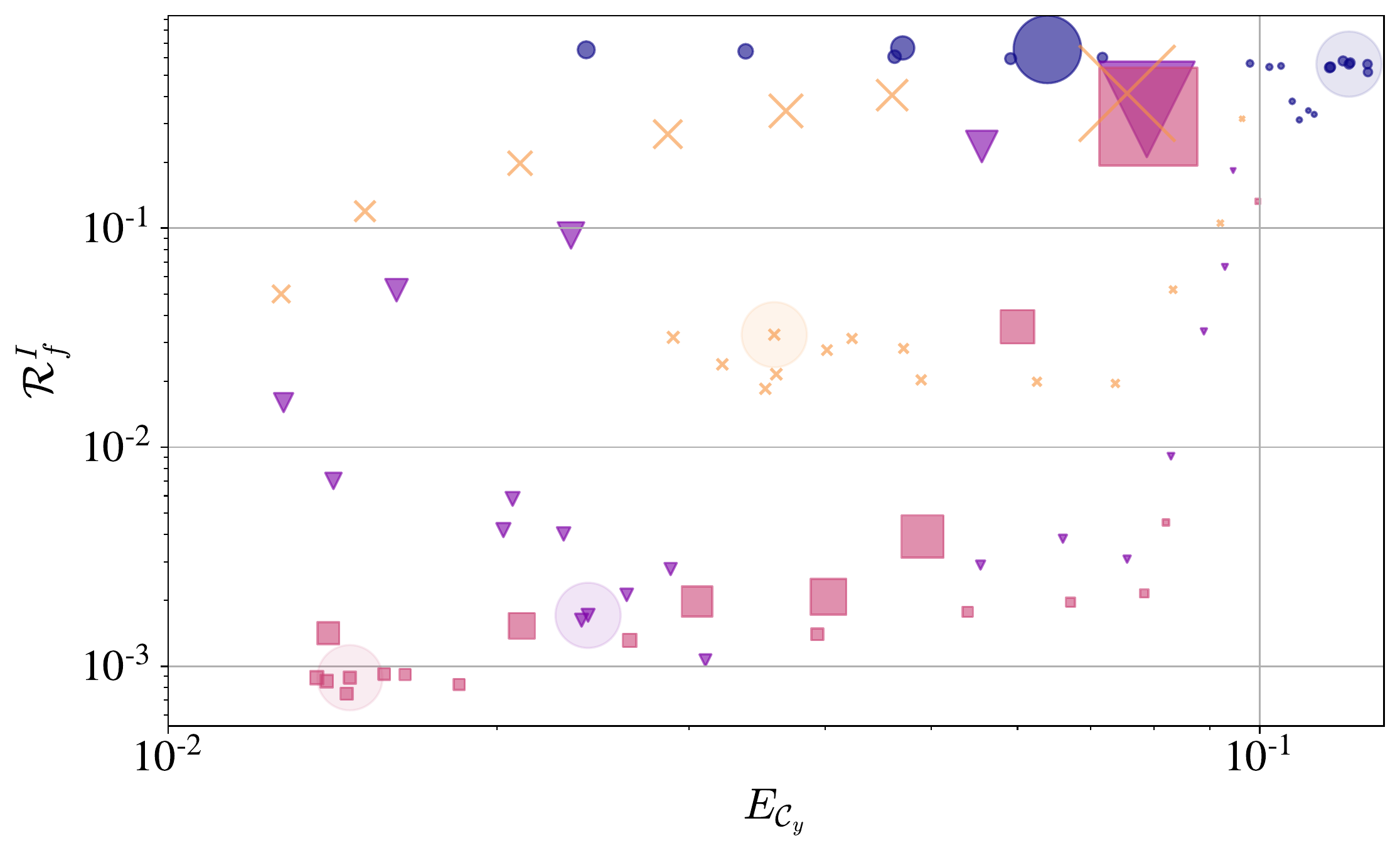}
    \caption{$E_{C_y}$ vs. $\mathcal{R}_f^{I}$ for various perturbations.}
    \end{subfigure}\\
    \caption{The effect on test calibration and regret when disturbing the posterior predictive uncertainty by $c\cdot \sigma(\xb)$ during the BO protocol. (a) Shows the overall calibration error of each model when a perturbation of $c\cdot \sigma(\xb)$ is done in each iteration, (b) shows the corresponding instant regret, and (c) depicts how regret and calibration varies together for the same experiments where the size of the points reflect its relative predictive sharpness (small markers being sharper than big markers and circled points are when $c=1$). \label{fig:regret-calibration-std-change}}
\end{figure*}

\paragraph{Disturbances in predictive uncertainties}
We finally study the effect of explicitly manipulating the predictive uncertainty of each model during the BO protocol. In \cref{fig:regret-calibration-std-change} we show the change in calibration (a) and regret (b) as a function of multiplying standard deviation of the predictive distribution by a constant $c$. Several interesting observations is contained in \cref{fig:regret-calibration-std-change}. First, all models exhibit the smallest calibration error at $c > 1$, which indicates some degree of overconfidence, and increasing the predictive variance slightly will in general lead to a better calibrated model. Second, all models are relatively robust to these uncertainty manipulations when it comes to the final regret obtained in the BO routine (b). Finally, in subfig. (c) we plot the calibration error together with regret for each of these manipulations $c$, where each marker is scaled with the corresponding mean predictive sharpness (sharper being smaller). And we find a highly interesting pattern: it appears that for the same calibration error, sharper models are associated to lower regrets, which might indicate that it is beneficial for BO if the surrogate model is overconfident rather than underconfident. This is likely related to the exploitation and exploration trade-off of the EI acquisition function. However, we will leave a more extensive investigation of this phenomenon for future work.

\vspace{-0.1cm}

\section{Discussion and Summary \label{sec:discussion}}

\vspace{-0.1cm}

In the previous section, we described and performed a number of numerical experiments to analyze the relationship between calibration and regret for BO. In this section, we will summarize some of the key take-aways as well as expand the analysis with a theoretical perspective.
\\
\\
\textbf{Take-away 1: Gaussian processes and Deep Ensembles work well for BO.} Our results for synthetic data is consistent with the apparent consensus that GPs are strong surrogates for BO and that they outperform the competing methods in terms of regret (both total and instant) (see \cref{correlation_table}. In the experiments with hyperparameter tuning of a neural network, DE and GP were superior to the competing methods in terms of regret. Here it is worth the highlight that the deep ensembles achieved slightly smaller regrets on average than the GPs, but the GPs are faster and simpler to train. In both experiments, the mean-field BNN performed significantly worse than all other methods, including random search. Similar behaviour has also been observed in other experimental design settings, e.g. active learning \citep{foong2019expressiveness}. In terms of model calibration, the GPs performed slightly better than the RFs and DEs for both synthetic and real data, but the difference is more pronounced in the hyperparameter tuning experiment (see \cref{correlation_table} and \cref{mnist_table}). Again, we notice that the mean-field BNNs are inferior to the other methods in both experiments.
\\
\\
\textbf{Take-away 2: The choice of model family is more important than calibration for BO performance.}
For the synthetic data, our analysis showed a positive correlation ($\rho \approx 0.28)$ between instant regret and calibration error, when computed across all dimensions, seeds and surrogates models. This suggests that models with low calibration errors are generally associated with low regret, i.e. strong BO performance (see \cref{correlation_table} and \cref{fig:regret-calibration-correlation}). Surprisingly, when we control for the type of surrogate model, the correlation vanishes (see table S1 in supplement). That is, within each model family, BO trials with lower calibration error is generally not associated with better BO performance. 
\\
\\
\begin{figure*}[t!]
    \includegraphics[width=\textwidth]{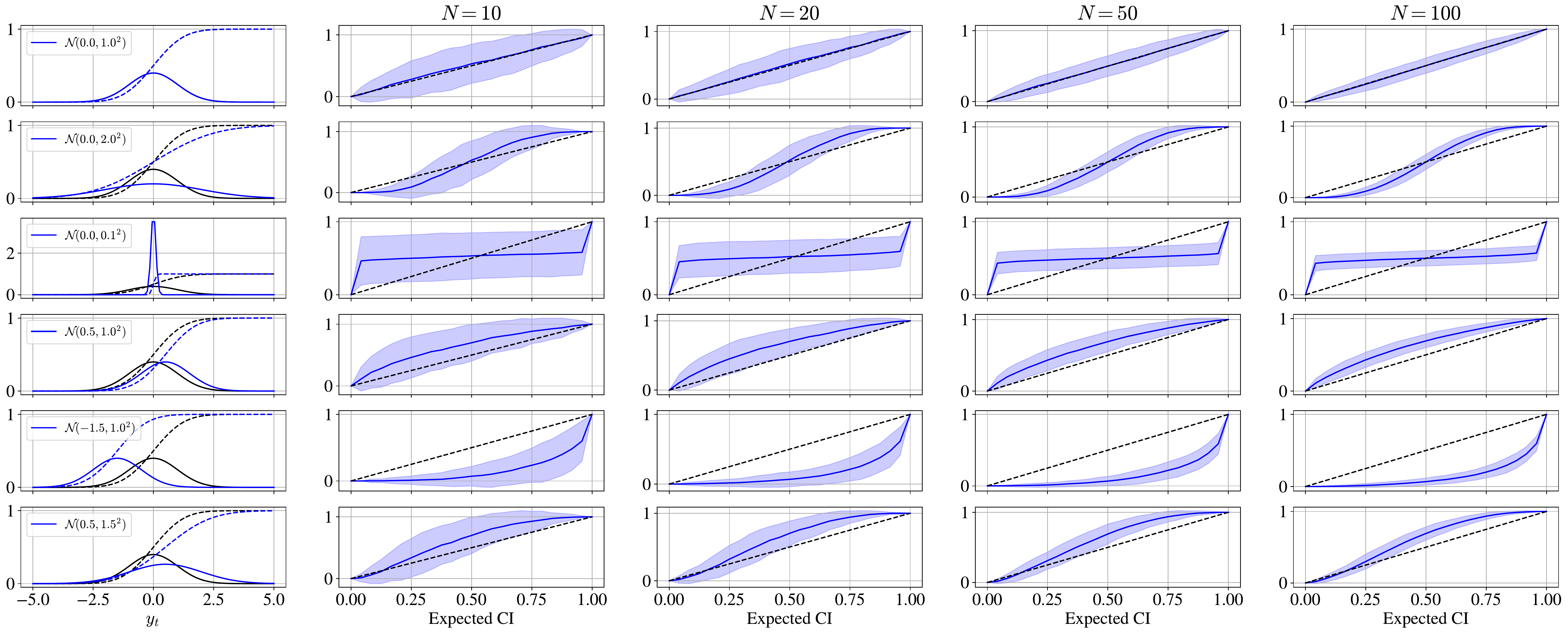}
    \caption{Examples of calibration curves computed on various number of test examples $N$, when the true data comes from a standard Gaussian and the model (left plots) varies (each row). Even in the best case scenario when samples are i.i.d., a large sample-to-sample variance can be expected in the ranges of $N$ for which BO normally operates. Calibration curve distributions are made from 100 random seeds, and the intervals corresponds to two times the standard deviation. \label{fig:n-vs-variance}}
\end{figure*}

\textbf{Take-away 3: BO performance is robust to minor changes in calibration, but sufficiently bad calibration leads to bad BO performance.}
To investigate this relationship further, we performed an intervention experiment, where we explicitly perturbed the posterior predictive distributions of the models to manipulate the degrees of calibration (see \cref{fig:regret-calibration-std-change}). More specifically, we scaled the standard deviation of all the posterior predictive distributions by a constant $c > 0$, where $0 < c < 1$ leads to more confident models, $c > 1$ leads to less confident models, and $c = 1$ is the baseline. \cref{fig:regret-calibration-std-change}(b) indicates that the GPs are more robust to such perturbations compared to the other methods and performed uniformly better than the other methods for all values of tested values of $c \in \left[10^{-2}, 10^2\right]$. 
The figure also shows that DEs and GPs achieve optimal regret for $c = 1$, indicating that the models are indeed calibrated 'enough' to carry out the BO task. 
The scatter in \cref{fig:regret-calibration-std-change}(c) plots the average instant regret as a function of calibration. The size of markers in the plot is controlled by the sharpness of the predictive distributions, and here it is clearly seen that overconfident models are generally performing better than underconfident models. 
Here we also observed a positive correlation, but interestingly, this figure also indicates that BO performance is fairly robust to the calibration error. 
That is, perturbing the predictive distribution by a small amount does not severely harm the BO performance and an extremely low calibration error is not required for decent BO performance. 
On the other hand, the figure also reveals that if the calibration become sufficiently bad, then BO performance drops quickly. 
Specifically, in this experiment, when the calibration error exceeded $\approx 0.08$, then the regret for all methods increased rapidly. 
Finally, the figure also suggests that the way a given model is miscalibrated matters. 
These observations can probably be explained by the fact that EI behaves more and more like random search, when the models become more and more uncertain.
\\
\\
\textbf{Take-away 4: Re-calibration does generally not improve BO performance.}
We further investigated the potential benefit of re-calibrating the surrogate models during the BO process. In our experiments with synthetic data, the re-calibration procedure leads to improved the calibration metric for all methods, however, the re-calibration did not improve regret for any methods. In fact, the regret increased slightly after re-calibration for several models. For the hyperparameter tuning experiment, re-calibration showed a minor positive effect on regret for DEs and RFs, but it is not clear whether the increase is due to random chance or actually an effect of re-calibration since the improvements are small relatively to the estimation uncertainty. Therefore, our results indicate that re-calibration does generally not improve regret, but we note that it does indeed introduce more variability in the results. We investigate these observations further from a theoretical perspective shortly. If the only goal of re-calibration is to achieve a better exploration/exploitation strategy, it might be more beneficial to tune the exploration parameter $\xi$ rather than doing re-calibration. However, there might be situations where one does not only wish to achieve better regret, but also learn about the underlying function while doing BO; and in such cases, re-calibration might be more relevant.

\textbf{Hypothesis: Calibration curves are not reliable for small sample sizes.}
Recent work by \cite{deshpande2021calibration} observe that re-calibration might aid BO by yielding smaller total regret in some trials and smaller instant regret for the BO last iteration in fewer trials. However, our experiments indicates that re-calibration does not improve BO performance in general and in fact, can degrade BO performance. Noting that important prior work states that re-calibration is feasible given  "enough i.i.d." data \cite{kuleshov2018accurate}, we hypothesize that this somewhat surprising observation in \cite{deshpande2021calibration} can be explained by the small sample sizes typically used in BO. Furthermore, we also note that the sample collection during BO is not i.i.d. due to the sequential nature of BO algorithms. In the following, we investigate this hypothesis using both theoretical and empirical analysis. 
For this analysis, our starting point will be a simple regression setting, where $p_y(y|x)$ denotes the true data generating distribution of $y$ given an input $x$. We further assume a trained model with predictive distribution  $p_t(y|x)$ aiming to mimic $p_y$ via training samples. Consider now the task of assessing the calibration of model using a set of i.i.d. validation samples $\{y_1,y_2,...,y_N\}$. 
Given the typical sample sizes used in BO, a natural question to ask is how accurate can we asses the calibration curve as a function of the size of the validation set $N$? 
We illustrate this in \cref{fig:n-vs-variance}, where the true data generating distribution $p_y(y|x) = \mathcal{N}(y|0,1)$ is approximated by six different model distributions $p_t(y|x)$ (one for each row). The first column shows the PDF and CDF of the true distribution and the model distribution in blue and black, respectively. Each of the subsequent columns shows the estimated calibration curves as a function of the number validation samples $N$. We repeat this experiment one hundred times and display the mean and confidence intervals corresponding to $\pm 2$ standard deviations. 

\textbf{Theoretical analysis of empirical calibration}
As expected, the sampling distributions for the calibration curves are wide for small samples and the variance shrinks with the number of samples, but how fast does the variance decay? We provide a simple yet powerful theoretical statement:
\begin{proposition}\label{prop1}
Let $F_i$ be the CDF of the predictive distribution for the $i$'th observation and let $\{ y_i \}_{i=1}^N$ be i.i.d. samples $y_i \sim p_y$. For $\mathcal{C}_y(p) = \frac{1}{N} \sum_{i=1}^N \mathbb{I}\left[y_i \leq F_i^{-1}(p)\right]$, then the variance of $C_y(p)$ decays as $\mathbb{V}\left[\mathcal{C}_y(p)\right] = \mathcal{O}(N^{-1})$.
\label{thm:var-vs-n}
\end{proposition}
\begin{proof}
Let $\mathcal{C}_y(p) = \frac{1}{N} \sum_{i=1}^N z_i$ for $z_i \equiv \mathbb{I}\left[y_i \leq F_i^{-1}(p)\right]$. The variance of $\mathcal{C}_y(p)$ is then given by 
\begin{align*}
\mathbb{V}\left[\mathcal{C}_y(p) \right] &= \mathbb{V}\left[\frac{1}{N} \sum_{i=1}^N z_i\right]
\end{align*}
by independence each $z_i$, and
\begin{align*}
\mathbb{V}\left[\mathcal{C}_y(p) \right] &\leq \frac{1}{N^2} \sum_{i=1}^N \sup\limits_i \mathbb{V}\left[z_i\right]
= \frac{1}{N^2} \sum_{i=1}^N \frac{1}{2^2}= \frac{1}{N} \frac{1}{2^2}.
\end{align*}
Hence, it follows the variance of $\mathcal{C}_y(p)$ is bounded by
\begin{align}
\mathbb{V}\left[\mathcal{C}_y(p)\right] \leq \mathcal{O}\left(N^{-1}\right).
\end{align}
See supplementary material for detailed proof.
\end{proof}
To validate this finding, we now expand the experiment from \cref{fig:n-vs-variance}. 
In \cref{fig:std_calibration_n}, we have conducted a numerical experiment, where we sample 100 models of the form $p_t(y|x) = \mathcal{N}(y|\mu, \sigma)$, where $\mu \sim \mathcal{N}(0, 1)$ and $\sigma \sim \text{LogNormal}(1,1)$.
For each model, we compute 100 calibration curves for each sample size $N \in [5,1000]$ and subsequently estimate the variance of those curves. \cref{fig:std_calibration_n} shows the maximum standard deviation as a function of the sample size $N$.
Perfectly consistent with the predictions from \cref{prop1}, we observe that maximum standard deviation decays proportional to $ \frac{1}{\sqrt{N}}$. We note here that our result in \cref{thm:var-vs-n} is irrespective of both the data and model distributions, and hence, the result is very general.%, as long as they are continuous and have a well-defined CDF. 
\begin{figure}[t!]
\centering
    \includegraphics[width=0.6\columnwidth]{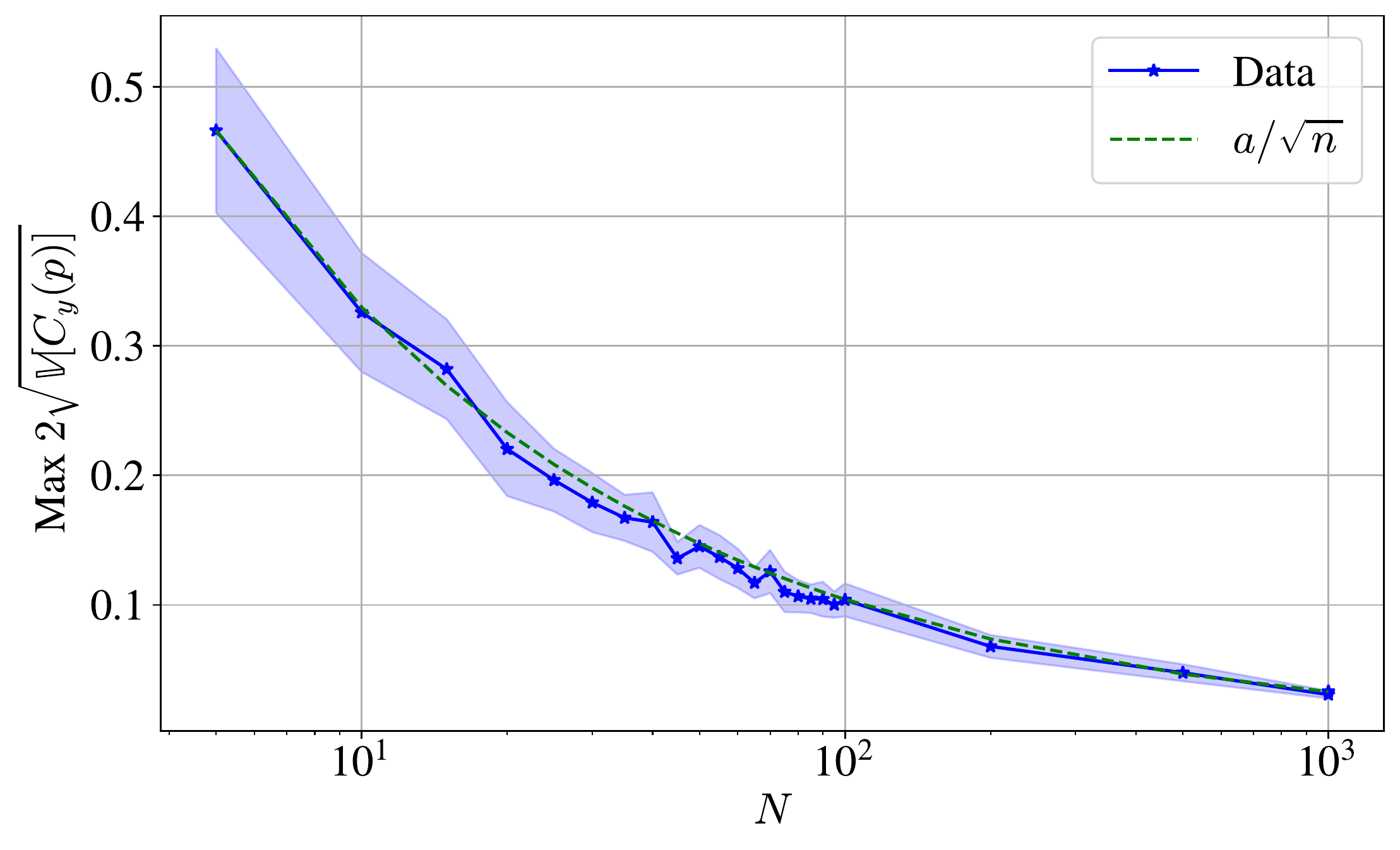}
    \caption{Maximum uncertainty across $p$ for calibration distribution $C_p(y)$ when $N$ samples of $y$ is given for computing the individual calibration curves. We sample 100 models (normal distributions) each with arguments $\mu_i \sim \text{Normal}(0,1)$ and $\sigma_i \sim \text{LogNormal}(1,1)$ each modelling data coming from a standard normal. For each experiment 100 calibration curves, that is 100 independent samples of size $N$ from the true model, constitutes the mean and std. We also plot the function $f(N) = a/\sqrt{N}$ for $a\approx 1.05$. \label{fig:std_calibration_n}}
\end{figure}
Next, we investigate the special case,  where the model is perfect, i.e. $p_t(y|x) = p_y(y|x)$:

\begin{proposition}
Let $F_i$ be the CDF of the predictive distribution perfect model, i.e. $p_t(y|x) = p(y|x)$. If $F_i$ is strictly monotonic, it holds that $\mathbb{V}\left[\mathcal{C}_y(p)\right] = \frac{p(1-p)}{N}$ for all $p$.
\label{prop2}
\end{proposition}
\begin{proof}
In this setting, we have
\begin{align*}
z_i = \mathbb{I}\left[y_i \leq F_i^{-1}(p)\right] =\mathbb{I}\left[F_i(y_i) \leq p\right] = \mathbb{I}\left[u_i \leq p\right],
\end{align*}
where $u_i \sim \mathcal{U}\left[0, 1\right]$ are uniformly distributed on the unit interval due to the probability integral transform. Since $\{ u_i \}_{i=1}^N$ are also independent, it follows that  $S_n = \sum_{i=1}^N z_i \sim \text{Binomial}(N, p).$
%\begin{align*}
%S_n = \sum_{i=1}^N z_i \sim \text{Binomial}(N, p).
%\end{align*}
Therefore, we have
\begin{align*}
\mathbb{V}\left[\mathcal{C}_y(p)\right] = \mathbb{V}\left[N^{-1}S_N\right] = N^{-2} \mathbb{V}\left[S_N\right] = N^{-1}p(1-p).
\end{align*}
This completes the proof.
\end{proof}

As seen in the top row of \cref{fig:n-vs-variance}, substantial variability should be expected for small sizes, and consequently, one may erroneously end up concluding that even a perfect model is miscalibrated for sufficiently small $N$. To quantify this relationship futher, we provide a theoretical result for the calibration error in eq. \eqref{eq:cal_error} in this setting. 

\begin{proposition}
Let $E_c = \sum_{j=1}^P w_j (p_j - \mathcal{C}_y(p_j))^2$ be the weighted mean square calibration error. Assume $w_i \in \left[0, 1\right]$ and $0 < p_1 < p_2 < ... < p_P < 1$ are fixed, and assume the CDF of the predictive distribution is equal to the true data distribution (almost everywhere), then it holds that $\mathbb{E}\left[E_c\right] = \frac{1}{P}\sum_{j=1}^P w_jp_j(1-p_j) = \mathcal{O}(N^{-1})$.
\label{thm:cal-err}
\end{proposition}
\begin{proof}
See supplementary material.
\end{proof}

\textbf{Take-away 6: Calibration curves are not reliable for small sample sizes}
\cref{prop1} and \cref{prop2} states that the variance of the estimator of the empirical calibration decreases with $\mathcal{O}\left(N^{-1}\right)$. This implies that empirical calibration curves are likely to be unreliable for small sample sizes and to improve the accuracy of the estimates by a factor of $10$, one needs to increase the size of the validation set by a factor of $100$, which will often be infeasible in practical BO settings. Furthermore, \cref{thm:cal-err} states that even for a perfect model, the expected calibration error is $\mathcal{O}\left(N^{-1}\right)$. Therefore, for small sample size, one should be careful concluding that a model is mis-calibrated, since the observed calibration error might as well be caused by the sample size. Even worse, when performing re-calibration in this scenario, one might risk adjusting the model in the "wrong direction" causing the model to be more miscalibrated than the original model. 

Although our empirical and theoretical analysis are focused on simple i.i.d. data/models, we expect the effect to be even more severe in the non-i.i.d. case since the effective sample size is typically smaller for correlated samples \citep{thiebaux1984interpretation}. Therefore, we claim that these effects may have profound impact on recalibration in BO protocols.

\paragraph{Future work}
Our study indicates that the common way to diagnose calibration (on a large test set) might not be sensible for BO and that future studies about calibration metrics more relevant to BO are needed. If the aleatoric noise is important for the BO task (and it is at least quasi-homoscedastic), it might be beneficial to spend part of the sampling budget on learning the noise level and subsequently using this as part of the surrogate model. Lastly, we note that it might be interesting to dig deeper into the effects of under- vs. over confidence on BO performance. %Future studies could investigate more acquisitions functions as well as richer surrogate functions.
 
\bibliography{references} 
\section*{SUPPLEMENTARY MATERIAL}
\subsection*{Mathematical Proofs}
\textbf{Proposition 1}:
Let $F_i$ be the CDF of the predictive distribution for the $i$'th observation and let $\{ y_i \}_{i=1}^n$ be i.i.d. samples $y_i \sim p_y$. For $\mathcal{C}_y(p) = \frac{1}{n} \sum_{i=1}^n \mathbb{I}\left[y_i \leq F_i^{-1}(p)\right]$, then the variance of $C_y(p)$ is bounded by $1/n$, i.e. $\mathbb{V}\left[C\right] = \mathcal{O}(n^{-1})$. 

\textbf{Proof:}
First, we show that the variance is bounded by $\mathcal{O}(n^{-1})$. We have
\begin{align}
\mathcal{C}_y(p) = \frac{1}{n} \sum_{i=1}^n \mathbb{I}\left[y_i \leq F_i^{-1}(p)\right]= \frac{1}{n} \sum_{i=1}^n z_i,
\end{align}

where $z_i \equiv \mathbb{I}\left[y_i \leq F_i^{-1}(p)\right]$. The variance of $\mathcal{C}_y(p)$ is then by give

\begin{equation}
    \begin{split}
       \mathbb{V}\left[\mathcal{C}_y(p) \right] &= \mathbb{V}\left[\frac{1}{n} \sum_{i=1}^n z_i\right]\\
&= \frac{1}{n^2}\mathbb{V}\left[ \sum_{i=1}^n z_i\right]\\
&= \frac{1}{n^2} \sum_{i=1}^n \mathbb{V}\left[z_i\right]\\
&\leq \frac{1}{n^2} \sum_{i=1}^n \sup\limits_i \mathbb{V}\left[z_i\right]\\
&\leq \frac{1}{n^2} \sum_{i=1}^n \frac{1}{2^2}\\
&= \frac{1}{n} \frac{1}{2^2} 
    \end{split}
\end{equation}

Hence, it also follows the standard deviation of $\mathcal{C}_y(p)$ is bounded by
\begin{align}
\sqrt{\mathcal{C}_y(p)} \leq \sqrt{\frac{1}{n} \frac{1}{2^2}} = \frac{1}{2\sqrt{n}}  = \mathcal{O}\left(\frac{1}{\sqrt{n}}\right).
\end{align}
This completes the proof of the first statement. 

\textbf{Lemma 1}:
Given a perfectly calibrated model, it holds that $\mathbb{V}\left[\mathcal{C}_y(p)\right] = \frac{p(1-p)}{n}$ for all $p$.

\textbf{Proof:}
In this setting, we have
\begin{align}
z_i = \mathbb{I}\left[y_i \leq F_i^{-1}(p)\right] =\mathbb{I}\left[F_i(y_i) \leq p\right] = \mathbb{I}\left[u_i] \leq p\right],
\end{align}
where $u_i \sim \mathcal{U}\left[0, 1\right]$ are uniformly distributed on the unit interval due to the probability integral transform. Since $\{ u_i \}_{i=1}^n$ are also independent, it follows that  
\begin{align}
S_n = \sum_{i=1}^n z_i \sim \text{Binomial}(n, p).
\end{align}

Therefore, it follows that
\begin{equation}
    \begin{split}
        \mathbb{V}\left[\mathcal{C}_y(p)\right] &= \mathbb{V}\left[\frac{1}{n}S\right] = \frac{1}{n^2} \mathbb{V}\left[S\right] \\
        & = \frac{1}{n^2}np(1-p) = \frac{p(1-p)}{n}.
    \end{split}
\end{equation}

This completes the proof.
%\pagebreak

\textbf{Proposition 2}: Let $E_c = \sum_{j=1}^m w_j (p_j - \mathcal{C}_y(p_j))^2$ be the weighted mean square calibration error. Assume $w_i \in \left[0, 1\right]$ and $0 < p_1 < p_2 < ... < p_m < 1$ are fixed, and assume the CDF of the predictive distribution is equal to the true data distribution (almost everywhere), then it holds that $\mathbb{E}\left[E_c\right] = \frac{1}{n}\sum_{j=1}^m w_jp_j(1-p_j) = \mathcal{O}(n^{-1})$ if $y_i \sim p_y$ are i.i.d. samples.

The calibration error $E_C$ is defined as follows
\begin{align}
E_c &= \sum_{j=1}^m w_j (p_j - \mathcal{C}_y(p_j))^2,
\end{align}
where each $w_i \in \left[0, 1\right]$ is a weight and $0 \leq p_1 < p_2 < ... < p_m < 1$ is predefined set of points.

In order to compute the expectation of $E_C$, we first expand:
\begin{align}
E &= \sum_{j=1}^m w_j (p_j^2 + \mathcal{C}_y(p_j)^2 - 2 p_j \mathcal{C}_y(p_j)) \\
&=   \sum_{j=1}^m w_j\mathcal{C}_y(p_j)^2 - 2  \sum_{j=1}^m w_j p_j \mathcal{C}_y(p_j))
\end{align}

Then it follows that
\begin{align} \label{eq:prop2}
\mathbb{E_C}\left[E\right] &= \mathbb{E}\left[ \sum_{j=1}^m w_j p_j^2 +  \sum_{j=1}^m w_j\mathcal{C}_y(p_j)^2 - 2  \sum_{j=1}^m w_j p_j \mathcal{C}_y(p_j))\right]\\
&= \sum_{j=1}^m w_j p_j^2 +  \sum_{j=1}^m w_j \mathbb{E}\left[ \mathcal{C}_y(p_j)^2\right] - 2  \sum_{j=1}^m w_j p_j \mathbb{E}\left[ \mathcal{C}_y(p_j)\right].
\end{align}

The first moment evaluates to
\begin{align}
    \mathbb{E}[C_y(p)] &= \int_{-\infty}^{\infty}  \mathbb{I}[y_t \leq F_t^{-1}(p)] p_y \text{d} y \\
    &= \int_{-\infty}^{F_t^{-1}(p)} p_y \text{d} y\\
    &= F_y (F_t^{-1}(p)) \\
    &= p.
\end{align}

Similarly, the second moment evaluates to
\begin{align}
\mathbb{E}\left[\mathcal{C}_y(p)^2\right] &= \mathbb{E}\left[\left(\frac{1}{n} \sum_{i=1}^n z_i\right)^2\right]\\
&= \frac{1}{n^2} \mathbb{E}\left[\sum_{i=1}^n \sum_{j=1}^n z_i z_j\right]\\
&= \frac{1}{n^2} \sum_{i=1}^n \mathbb{E}\left[z_i^2\right] + \frac{1}{n^2}\sum_{j\neq i} \mathbb{E}\left[z_i z_j \right]\\
&= \frac{1}{n^2} \sum_{i=1}^n p + \frac{1}{n^2}\sum_{j\neq i} \mathbb{E}\left[z_i\right] \mathbb{E}\left[z_j \right]\\
&= \frac{n}{n^2} p + \frac{1}{n^2}\sum_{j\neq i} p^2\\
&= \frac{1}{n} p + \frac{1}{n^2}\left(n^2 - n\right) p^2
\end{align}

Rearranging the terms yields
\begin{equation}
    \begin{split}
        \mathbb{E}\left[\mathcal{C}_y(p)^2\right] %
&= \frac{1}{n} p + \frac{n^2 - n}{n^2} p^2\\
&= \frac{1}{n} p - \frac{1}{n} p^2 + p^2\\
&= \frac{p(1-p)}{n} + p^2
    \end{split}
\end{equation}

Substituting the moments into eq. \eqref{eq:prop2} yields
\begin{equation}
    \begin{split}
\mathbb{E}\left[E_C\right] &= \sum_{j=1}^m w_j p_j^2 +  \sum_{j=1}^m w_j \left[\frac{p_j(1-p_j)}{n} + p_j^2\right]  \\
&- 2  \sum_{j=1}^m w_j p_j^2\\
&= \sum_{j=1}^m w_j p_j^2 +  \sum_{j=1}^m w_j \frac{p_j(1-p_j)}{n}  \\
&+ \sum_{j=1}^m w_j p_j^2  - 2  \sum_{j=1}^m w_j p_j^2\\
&= \frac{1}{n}\sum_{j=1}^m w_jp_j(1-p_j)\\
&= \mathcal{O}(n^{-1}).
    \end{split}
\end{equation}

This completes the proof.

\subsection*{If $p_y$ and $p_t$ are normal distributions \label{sec:p_yt_normals}}    
For non-perfect models we have that $F_y (F_t^{-1}(p)) = g(p)$ where in general $g(p) \neq p$. If both $p_y$ and $p_t$ are normal distributions, the CDF and inverse CDF of a normal are, respectively, given by
\begin{align*}
    F(x) &= \frac{1}{2}\left[1 + \text{erf}\left(\frac{x-\mu}{\sigma \sqrt{2}}\right)\right] \\
    F^{-1}(p) &= \mu + \sigma \sqrt{2} \text{erf}^{-1}\left(2p-1\right) \\
\end{align*}
When data comes from $y_t \sim \mathcal{N}(\mu_y,\sigma_y^2)$ and the model is $\mathcal{N}(\mu_t,\sigma_t^2)$, we can write the expectation of the calibration curve as follows

\begin{align*}
    g(p) &= F_y(F_t^{-1}(p)) \\
    &=  \frac{1}{2}\left[1 + \text{erf}\left(\frac{F_t^{-1}(p)-\mu_y}{\sigma_y \sqrt{2}}\right)\right]  \\
    &=  \frac{1}{2}\left[1 + \text{erf}\left(\frac{\mu_t + \sigma _t\sqrt{2} \text{erf}^{-1}\left(2p-1\right)-\mu_y}{\sigma_y \sqrt{2}}\right)\right]  \\
    &=  \frac{1}{2}\left[1 + \text{erf}\left(\frac{\mu_t -\mu_y}{\sigma_y \sqrt{2}} + \frac{\sigma _t}{\sigma_y }\text{erf}^{-1}\left(2p-1\right)\right)\right]  \\
    &=  \frac{1}{2}\left[1 + \text{erf}\left(\frac{\mu_t -\mu_y}{\sigma_y \sqrt{2}} + \frac{\sigma _t}{\sigma_y }\text{erf}^{-1}\left(2p-1\right)\right)\right]  \\
\end{align*}

which also evaluates to $p$ for a perfect model:
\begin{align*}
    g(p) &= \frac{1}{2}\left[1 + \text{erf}\left(\frac{0}{\sigma_y \sqrt{2}} + 1 \cdot \text{erf}^{-1}\left(2p-1\right)\right)\right]  \\
    &= \frac{1}{2}\left[1 + 2p-1\right]  \\
    &= p  \\
\end{align*}

\begin{comment}
\subsection*{Additional Results}

\begin{table*}[h!]
\caption{CI95 for the linear model $\mathcal{R}_f^I = \alpha \mathcal{C}(y) + \beta \mathcal{S}(y) + \gamma \mathcal{L}(y) + \mu$ is fitted. CI95 are included on the OLS fitted coefficients with corresponding stars indicating significance. $^*= p<0.05$, $^{**} = p<10^{-10}$ } 
\begin{center}
\begin{tabular}{lcccc}
& $ \alpha \cdot \mathcal{C}_R(y)$ & $\beta\cdot \mathcal{S}_R(y)$ & $\gamma\cdot \mathcal{L}_R(y)$ & \\
\midrule 
BNN&$(-0.23,-0.11)^{*}$&$(-0.26,-0.03)^{*}$&$(-0.14,0.08)$&  \\
DE&$(-0.01,0.01)$&$(-0.01,0.01)^{*}$&$(-0.01,0.01)$& \\
GP&$(-0.01,0.01)$&$(-0.01,0.01)$&$(-0.01,0.01)$& \\
RF&$(-0.02,0.02)$&$(-0.04,0.01)$&$(-0.018,0.02)$& \\
All&$(-0.02,0.08)$&$(-0.33,-0.24)^{**}$&$(-0.11,-0.04)^{*}$& \\
\midrule
& $ \alpha \cdot \mathcal{C}_{BO}(y)$ & $\beta\cdot \mathcal{S}_{BO}(y)$ & $\gamma\cdot \mathcal{L}_{BO}(y)$ &  \\
\midrule 
BNN&$(-0.34,-0.26)^{**}$&$(-0.39,-0.24)^{**}$&$(-0.20,-0.06)^{*}$&\\
DE&$(0.10,0.19)^{*}$&$(-0.03,0.07)$&$(-0.12,-0.02)^{*}$&\\
GP&$(-0.13,-0.05)^{*}$&$(-0.40,-0.32)^{**}$&$(-0.09,-0.01)^{*}$&\\
RF&$(-0.19,-0.11)^{**}$&$(-0.2,-0.11)^{**}$&$(-0.10,-0.02)^{*}$&\\
All&$(-0.279,-0.183)^{**}$&$(-0.090,-0.007)$&$(-0.046,0.014)$& \\
\end{tabular}
\end{center}
\end{table*}
\end{comment}

\end{document}